\documentclass{article}

\def\arxivmode{1}

\usepackage{amsmath,amsbsy,amssymb,amsfonts,amsthm}
\usepackage{nicefrac}
\usepackage{mathtools}
\usepackage{color}
\usepackage{xspace} %
\usepackage[numbers]{natbib} %
\usepackage{times}
\usepackage{graphicx,caption,subcaption}
\usepackage{algorithm,algorithmic} 
\usepackage{hyperref}
\usepackage{breakurl}

\ifdefined\arxivmode
	\usepackage{modified_nips15submit_e,times}
	\nipsfinalcopy
\else
	\usepackage{nips15submit_e,times}
	\nipsfinalcopy
\fi

\newcommand{\smiddle}{\mathrel{}\middle|\mathrel{}} %

\newcommand{\eps}{\epsilon}

\def\balign#1\ealign{\begin{align}#1\end{align}}
\def\baligns#1\ealigns{\begin{align*}#1\end{align*}}
\def\balignat#1\ealign{\begin{alignat}#1\end{alignat}}
\def\balignats#1\ealigns{\begin{alignat*}#1\end{alignat*}}
\def\bitemize#1\eitemize{\begin{itemize}#1\end{itemize}}
\def\benumerate#1\eenumerate{\begin{enumerate}#1\end{enumerate}}

\newenvironment{talign*}
 {\csname align*\endcsname}
 {\endalign}
\newenvironment{talign}
 {\csname align\endcsname}
 {\endalign}

\def\balignst#1\ealignst{\begin{talign*}#1\end{talign*}}
\def\balignt#1\ealignt{\begin{talign}#1\end{talign}}%

\newcommand{\qtext}[1]{\quad\text{#1}\quad} 

\let\originalleft\left
\let\originalright\right
\renewcommand{\left}{\mathopen{}\mathclose\bgroup\originalleft}
\renewcommand{\right}{\aftergroup\egroup\originalright}

\def\Holder{H\"older\xspace}

\def\tinycitep*#1{{\tiny\citep*{#1}}}
\def\tinycitealt*#1{{\tiny\citealt*{#1}}}
\def\tinycite*#1{{\tiny\cite*{#1}}}

\def\smallcitep*#1{{\scriptsize\citep*{#1}}}
\def\smallcitealt*#1{{\scriptsize\citealt*{#1}}}
\def\smallcite*#1{{\scriptsize\cite*{#1}}}

\def\mbb#1{\mathbb{#1}}

\def\tsc#1{\textsc{#1}}

\def\textsum{{\textstyle\sum}} %
\def\reals{\mathbb{R}} %

\def\<{\left\langle} %
\def\>{\right\rangle}

\def\defeq{\triangleq} %
\def\norm#1{\left\|{#1}\right\|} %
\newcommand{\onenorm}[1]{\norm{#1}_1} %
\newcommand{\twonorm}[1]{\norm{#1}_2} %
\newcommand{\infnorm}[1]{\norm{#1}_{\infty}} %
\newcommand{\opnorm}[1]{\norm{#1}_{op}} %
\newcommand{\inner}[2]{\langle{#1},{#2}\rangle} %
\def\indic#1{\mbb{I}\left[{#1}\right]} %

\def\maxarg#1{\max\left({#1}\right)} %
\def\E{\mbb{E}} %
\def\Earg#1{\E\left[{#1}\right]}

\def\Esubarg#1#2{\E_{#1}\left[{#2}\right]}
\newcommand{\Gsn}{\mathcal{N}}

\newcommand{\Unif}{\textnormal{Unif}}

\newcommand{\grad}{\nabla} %
\newcommand{\Hess}{\nabla^2} %
\newcommand{\deriv}[2]{\frac{d #1}{d #2}} %
\newcommand{\pderiv}[2]{\frac{\partial #1}{\partial #2}} %

\providecommand{\sign}{\mathop\mathrm{sign}}

\def\supp#1{\mathrm{supp}({#1})}

\newcommand{\algref}[1]{Algorithm~{\ref{alg:#1}}}
\newcommand{\appref}[1]{Appendix~{\ref{sec:#1}}}

\newcommand{\eqnref}[1]{\eqref{eqn:#1}}

\newcommand{\figref}[1]{Figure~{\ref{fig:#1}}}
\newcommand{\lemref}[1]{Lemma~{\ref{lem:#1}}}

\newcommand{\secref}[1]{Section~{\ref{sec:#1}}}
\newcommand{\secsref}[1]{Sections~{\ref{sec:#1}}}
\newcommand{\secssref}[1]{{\ref{sec:#1}}}
\newcommand{\propref}[1]{Proposition~{\ref{prop:#1}}}

\newcommand{\thmref}[1]{Theorem~{\ref{thm:#1}}}
\ifdefined\nonewproofenvironments\else
\ifdefined\ispres\else
\newtheorem{theorem}{Theorem}
\newtheorem{lemma}[theorem]{Lemma}

\renewenvironment{proof}{\noindent\textbf{Proof}\hspace*{1em}}{\qed\\}
\newenvironment{proof-sketch}{\noindent\textbf{Proof Sketch}
  \hspace*{1em}}{\qed\bigskip\\}
\newenvironment{proof-idea}{\noindent\textbf{Proof Idea}
  \hspace*{1em}}{\qed\bigskip\\}
\newenvironment{proof-of-lemma}[1][{}]{\noindent\textbf{Proof of Lemma {#1}}
  \hspace*{1em}}{\qed\\}
\newenvironment{proof-of-theorem}[1][{}]{\noindent\textbf{Proof of Theorem {#1}}
  \hspace*{1em}}{\qed\\}
\newenvironment{proof-attempt}{\noindent\textbf{Proof Attempt}
  \hspace*{1em}}{\qed\bigskip\\}

\fi

\newtheorem{proposition}[theorem]{Proposition}

\fi
\newcommand{\xset}[0]{\mathcal{X}} %
\newcommand{\gset}[0]{\mathcal{G}} %
\newcommand{\steinset}[0]{\steinsetarg{}} %
\newcommand{\steinsetarg}[1]{\gset_{\norm{\cdot}_{#1}}} %
\newcommand{\gsteinset}[2]{\gset_{\norm{\cdot}_{#1},Q,{#2}}} %
\newcommand{\hset}[0]{\mathcal{H}} %

\newcommand{\wasssetarg}[1]{\mathcal{W}_{#1}} %
\newcommand{\wassset}{\wasssetarg{\norm{\cdot}}} %
\newcommand{\smoothset}{\mathcal{M}_{\norm{\cdot}}} %

\newcommand{\ball}{\mathcal{B}} %

\newcommand{\generator}[1]{\mathcal{A}{#1}} %
\newcommand{\genarg}[2]{(\generator{#1})({#2})} %

\newcommand{\operator}[1]{\mathcal{T}{#1}} %
\newcommand{\oparg}[2]{(\operator{#1})({#2})} %

\newcommand{\langevin}[1]{\mathcal{T}_P{#1}} %

\newcommand{\langarg}[2]{(\langevin{#1})({#2})} %
\newcommand{\langgenarg}[2]{(\mathcal{A}_P{#1})({#2})} %

\newcommand{\ipm}{d_\hset} %
\newcommand{\stein}[3]{\mathcal{S}({#1},{#2},{#3})} %
\newcommand{\opstein}[2]{\stein{#1}{\operator{}}{#2}} %
\newcommand{\langstein}[2]{\stein{#1}{\langevin{}}{#2}} %
\newcommand{\wass}{d_{\wassset}} %
\newcommand{\smooth}{d_{\smoothset}} %

\newcommand{\pvar}[0]{z}%
\newcommand{\PVAR}[0]{\MakeUppercase{\pvar}} %

\newcommand{\process}[2]{\PVAR_{#1,#2}} %
\title{Measuring Sample Quality with Stein's Method}

\author{
Jackson Gorham\\
Department of Statistics\\
Stanford University\\
\And
Lester Mackey\\
Department of Statistics\\
Stanford University\\
}

\begin{document} 

\captionsetup{belowskip=0pt,aboveskip=4pt} %
\floatsep=\baselineskip %
\textfloatsep=\baselineskip %

\maketitle

\graphicspath{{./figs/}}

\begin{abstract}
To improve the efficiency of Monte Carlo estimation, practitioners are turning to biased Markov chain Monte Carlo procedures that trade off asymptotic exactness for computational speed.
The reasoning is sound: a reduction in variance due to more rapid sampling can outweigh the bias introduced.
However, the inexactness creates new challenges for sampler and parameter selection, since standard measures of sample quality like effective sample size do not account for asymptotic bias.
To address these challenges, we introduce a new computable quality measure based on Stein's method that quantifies the maximum discrepancy between sample and target expectations over a large class of test functions.
We use our tool to compare exact, biased, and deterministic sample sequences and illustrate applications to hyperparameter selection, convergence rate assessment, and quantifying bias-variance tradeoffs in posterior inference.
\end{abstract}

\section{Introduction}

When faced with a complex target distribution, one often turns to Markov chain Monte Carlo (MCMC) \citep{BrooksGeJoMe11} to approximate intractable expectations $\Esubarg{P}{h(Z)} = \int_{\xset} p(x)h(x) dx$ 
with asymptotically exact sample estimates $\Esubarg{Q}{h(X)} = \sum_{i=1}^n q(x_i)h(x_i)$.
These complex targets commonly arise as posterior distributions in Bayesian inference 
and as candidate distributions in maximum likelihood estimation \citep{Geyer91}.
In recent years, researchers \cite[e.g.,][]{WellingTe11,Ahn2012,Korattikara2014} have introduced asymptotic bias into MCMC procedures to trade off asymptotic correctness for improved sampling speed.
The rationale is that more rapid sampling can reduce the variance of a Monte Carlo estimate and hence outweigh the bias introduced.
However, the added flexibility introduces new challenges for sampler and parameter selection, since standard sample quality measures, like effective sample size, asymptotic variance, trace and mean plots, and 
pooled and within-chain variance diagnostics, presume eventual convergence to the target~\cite{BrooksGeJoMe11} and hence do not account for asymptotic bias.

To address this shortcoming, we develop a new measure of sample quality suitable for comparing asymptotically exact, asymptotically biased, and even deterministic sample sequences.
The quality measure is based on Stein's method and is attainable by solving a linear program.
After outlining our design criteria in \secref{quality}, we relate the convergence of the quality measure to that of standard probability metrics in \secref{steins-method}, develop a streamlined implementation based on geometric spanners in \secref{programs}, 
and illustrate applications to hyperparameter selection, convergence rate assessment, and the quantification of bias-variance tradeoffs in posterior inference in \secref{experiments}. 
We discuss related work in \secref{conclusions} and defer all proofs to the appendix.

\textbf{Notation}\quad 
We denote the $\ell_2$, $\ell_1$, and $\ell_\infty$ norms on $\reals^d$ by $\norm{\cdot}_2$, $\norm{\cdot}_1$, and $\norm{\cdot}_\infty$ respectively.
We will often refer to a generic norm $\norm{\cdot}$ on $\reals^d$ with associated dual norms $\norm{w}^*  \defeq \sup_{v\in \reals^d: \norm{v}=1}{\inner{w}{v}}$ for vectors $w\in\reals^d$,
$\norm{M}^* \defeq \sup_{v\in \reals^d: \norm{v}=1}{\norm{Mv}^*}$ for matrices $M\in \reals^{d \times d}$, 
and  $\norm{T}^* \defeq \sup_{v\in \reals^d: \norm{v}=1}{\norm{T[v]}^*}$ for tensors $T \in \reals^{d\times d\times d}$.
We denote the $j$-th standard basis vector by $e_j$, the partial derivative $\pderiv{}{x_k}$ by $\grad_k$,
 and the gradient of any $\reals^d$-valued function $g$ by $\grad g$ with components $(\grad g(x))_{jk} \defeq \grad_k g_j(x)$. %

\section{Quality Measures for Samples}
\label{sec:quality}

Consider a target distribution $P$ with convex support $\xset \subseteq \reals^d$ %
and continuously differentiable density $p$.
We assume that $p$ is known up to its normalizing constant and that exact integration under $P$ is intractable for most functions of interest.
We will approximate expectations under $P$ with the aid of a \emph{weighted sample}, a collection of distinct sample points $x_1, \dots, x_n \in \xset$ with weights $q(x_i)$ encoded in a probability mass function $q$.
The probability mass function $q$ induces a discrete distribution $Q$ and an approximation $\Esubarg{Q}{h(X)} = \sum_{i=1}^n q(x_i)h(x_i)$ for any target expectation $\Esubarg{P}{h(\PVAR)}$.
We make no assumption about the provenance of the sample points; they may arise as random draws from a Markov chain or even be deterministically selected.

Our goal is to compare the fidelity of different samples approximating a common target distribution.
That is, we seek to quantify the discrepancy between $\E_{Q}$ and $\E_P$ in a manner that (i) detects when a sequence of samples is converging to the target, (ii) detects when a sequence of samples is not converging to the target, and (iii) is computationally feasible.
We begin by considering
the maximum deviation between sample and target expectations over a class of real-valued test functions $\hset$,
\balign\label{eqn:ipm}
	\ipm(Q,P) = \sup_{h \in \hset} |\Esubarg{Q}{h(X)} - \Esubarg{P}{h(\PVAR)}|.
\ealign
When the class of test functions is sufficiently large, the convergence of $\ipm(Q_m, P)$ to zero implies that the sequence of sample measures $(Q_m)_{m\geq 1}$ converges weakly to $P$.
In this case, the expression \eqref{eqn:ipm} is termed an \emph{integral probability metric} (IPM) \cite{Muller97}. 
By varying the class of test functions $\hset$, we can recover many well-known probability metrics as IPMs, including
the \emph{total variation distance}, generated by $\hset = \{h: \xset\to\reals \mid \sup_{x\in\xset} |h(x)| \leq 1\}$, 
and the \emph{Wasserstein distance} (also known as the Kantorovich-Rubenstein or earth mover's distance), $\wass$, generated by 
\[\textstyle\hset = \wassset\defeq\{h : \xset\to\reals \mid \sup_{x\neq y\in\xset} \frac{|h(x) - h(y)|}{\norm{x-y}} \leq 1 \}.\]
The primary impediment to adopting an IPM as a sample quality measure is that exact computation is typically infeasible when generic integration under $P$ is intractable.
However, we could skirt this intractability by focusing on classes of test functions with known expectation under $P$.
For example, if we consider only test functions $h$ for which $\Esubarg{P}{h(\PVAR)} = 0$, then the IPM value $\ipm(Q,P)$ is the solution of an optimization problem depending on $Q$ alone.
This, at a high level, is our strategy, but many questions remain. 
How do we select the class of test functions $h$?
How do we know that the resulting IPM will track convergence and non-convergence of a sample sequence (Desiderata (i) and (ii))?
How do we solve the resulting optimization problem in practice (Desideratum (iii))?
To address the first two of these questions, we draw upon tools from Charles Stein's method of characterizing distributional convergence.
We return to the third question in \secref{programs}.

\section{Stein's Method}
\label{sec:steins-method}

Stein's method \citep{Stein72} for characterizing convergence in distribution classically proceeds in three steps:
\benumerate
	\item 
	Identify a real-valued operator $\operator{}$ acting on a set $\gset$ of 
	$\reals^d$-valued\footnote{Scalar functions $g$ are more common in Stein's method, but we will find $\reals^d$-valued $g$ more convenient.}
	functions of $\xset$ for which
	\balign\label{eqn:kernel-condition}
		\Esubarg{P}{\oparg{g}{\PVAR}} = 0 \qtext{for all} g\in \gset.
	\ealign
	Together, $\operator{}$ and $\gset$ define the \emph{Stein discrepancy}, 
	\baligns
		\opstein{Q}{\gset}
			\defeq \sup_{g\in\gset} |\Esubarg{Q}{\oparg{g}{X}}|
			= \sup_{g\in\gset} |\Esubarg{Q}{\oparg{g}{X}} - \Esubarg{P}{\oparg{g}{\PVAR}}|
			= d_{\operator{}\gset}(Q,P),
	\ealigns
	an IPM-type quality measure with no explicit integration under $P$.
	\item 
	Lower bound the Stein discrepancy by a familiar convergence-determining IPM $\ipm$.
	This step can be performed once, in advance, for large classes of target distributions and 
	ensures that, for any sequence of probability measures $(\mu_m)_{m\geq 1}$, $\opstein{\mu_m}{\gset}$ converges to zero only if $\ipm(\mu_m, P)$ does (Desideratum (ii)).
	\item
	 Upper bound the Stein discrepancy by any means necessary to demonstrate convergence to zero under suitable conditions (Desideratum (i)).
	In our case, the universal bound established in \secref{discrepancy-upper-bound} will suffice.
\eenumerate
While Stein's method is typically employed as an analytical tool, we view the Stein discrepancy as a promising candidate for a practical sample quality measure.
Indeed, in \secref{programs}, we will adopt an optimization perspective and develop efficient procedures 
to compute the Stein discrepancy for any sample measure $Q$ and appropriate choices of $\operator{}$ and $\gset$. 
First, we assess the convergence properties of an equivalent Stein discrepancy in the subsections to follow.

\subsection{Identifying a Stein Operator}
The \emph{generator method} of \citet{Barbour88} provides a convenient and general means of constructing operators $\operator{}$ which produce mean-zero functions under $P$ \eqref{eqn:kernel-condition} .
Let $(\PVAR_t)_{t\geq 0}$ represent a Markov process with unique stationary distribution $P$.
Then the \emph{infinitesimal generator} $\generator{}$ of $(\PVAR_t)_{t\geq 0}$, defined by 
\[
	\genarg{u}{x} = \lim_{t\to0} {(\Earg{u(\PVAR_t) \mid \PVAR_0 = x} - u(x))/}{t} \qtext{for} u: \reals^d \to \reals,
\]
satisfies $\Esubarg{P}{\genarg{u}{\PVAR}} = 0$ under mild conditions on $\generator{}$ and $u$.
Hence, a candidate operator $\operator{}$ can be constructed from any infinitesimal generator.

For example, the \emph{overdamped Langevin diffusion}, defined by the stochastic differential equation
$
	d\PVAR_t = \frac{1}{2}\grad\log p(\PVAR_t) dt + dW_t
$
for $(W_t)_{t\geq 0}$ a Wiener process,
gives rise to the generator
\balign\label{eqn:generator}
\langgenarg{u}{x} = \frac{1}{2}\inner{\grad u(x)}{\grad\log p(x)} + \frac{1}{2}\inner{\grad}{\grad u(x)}.
\ealign
After substituting $g$ for $\frac{1}{2}\grad u$, we obtain the associated \emph{Stein operator}\footnote{The operator $\langevin{}$ has also found fruitful application in the design of Monte Carlo control variates \cite{OatesGiCh14}.}
\balign\label{eqn:generator-op}
\langarg{g}{x} \defeq \inner{g(x)}{\grad\log p(x)} + \inner{\grad}{g(x)}.
\ealign
The Stein operator $\langevin{}$ is particularly well-suited to our setting as it depends on $P$ only through the derivative of its log density and hence is computable even when the normalizing constant of $p$ is not.

If we let $\partial\xset$ denote the boundary of $\xset$ (an empty set when $\xset=\reals^d$) and $n(x)$ represent the outward unit normal vector to the boundary at $x$, then we may define
the \emph{classical Stein set}
\baligns
	\steinset 
		\defeq \bigg\{ g :\xset\to\reals^d \bigg|
		&\sup_{x\neq y\in \xset} \maxarg{\norm{g(x)}^*,\norm{\grad g(x)}^*, \frac{\norm{\grad g(x) - \grad g(y)}^*}{\norm{x-y}}} \leq 1\qtext{and}\\
		&\inner{g(x)}{n(x)} = 0, \forall x \in \partial\xset \text{ with $n(x)$ defined}
		\bigg\}
\ealigns
of sufficiently smooth functions satisfying a Neumann-type boundary condition.
The following proposition -- a consequence of integration by parts -- shows that $\steinset$ is a suitable domain for $\langevin{}$.
\begin{proposition} \label{prop:langevin-zero}
	If $\Esubarg{P}{\norm{\grad \log p(\PVAR)}}<\infty$, then
	$\Esubarg{P}{\langarg{g}{\PVAR}} = 0$ for all $g\in \steinset$.
\end{proposition}
Together, $\langevin{}$ and $\steinset$ form the \emph{classical Stein discrepancy} $\langstein{Q}{\steinset}$, our chief object of study.

\subsection{Lower Bounding the Classical Stein Discrepancy}
In the univariate setting ($d=1$), it is known for a wide variety of targets $P$ that
the classical Stein discrepancy $\langstein{\mu_m}{\steinset}$ converges to zero only if the Wasserstein distance $\wass(\mu_m,P)$ does \citep{ChenGoSh11,ChatterjeeSh11}.
In the multivariate setting, analogous statements are available for multivariate Gaussian targets \citep{ReinertRo09,ChatterjeeMe08,Meckes09}, but few other target distributions have been analyzed.
To extend the reach of the multivariate literature, we show in \thmref{concave-stein-lower-bound} that the classical Stein discrepancy also determines Wasserstein convergence for a large class of strongly log-concave densities, including the Bayesian logistic regression posterior under Gaussian priors.

\begin{theorem}[Stein Discrepancy Lower Bound for Strongly Log-concave Densities] \label{thm:concave-stein-lower-bound}
	If $\xset = \reals^d$, and $\log p$ is strongly concave with third and fourth derivatives bounded and continuous, 
	then, for any probability measures $(\mu_m)_{m\geq 1}$, $\langstein{\mu_m}{\steinset}\to 0$ only if $\wass(\mu_m,P)\to 0$.
\end{theorem}
We emphasize that the sufficient conditions in \thmref{concave-stein-lower-bound} are certainly not necessary for lower bounding the classical Stein discrepancy.
We hope that the theorem and its proof will provide a template for lower bounding $\langstein{Q}{\steinset}$ for other large classes of multivariate target distributions.

\subsection{Upper Bounding the Classical Stein Discrepancy} \label{sec:discrepancy-upper-bound}
We next  establish sufficient conditions for the convergence of the classical Stein discrepancy to zero.
\begin{proposition}[Stein Discrepancy Upper Bound] \label{prop:stein-upper-bound}
If $X \sim Q$ and $\PVAR \sim P$ with $\grad\log p(\PVAR)$ integrable,
\baligns
	&\langstein{Q}{\steinset}
		\leq \norm{I}\,\Earg{\norm{X-Z}} + \Earg{\norm{\grad\log p(X)-\grad\log p(\PVAR)}} + \Earg{\norm{\grad \log p(Z)(X-Z)^\top}} \\
		&\leq\ \norm{I}\,\Earg{\norm{X-Z}}+ \Earg{\norm{\grad\log p(X)-\grad\log p(\PVAR)}}\ +
		\textstyle\sqrt{\Earg{{\norm{\grad \log p(Z)}}^2}\Earg{\norm{X-Z}^2}}.
\ealigns
\end{proposition}
One implication of \propref{stein-upper-bound} is that $\langstein{Q_m}{\steinset}$ converges to zero whenever $X_m \sim Q_m$
converges in mean-square to $Z\sim P$ and $\grad \log p(X_m)$ converges in mean to $\grad \log p(Z)$.

\subsection{Extension to Non-uniform Stein Sets}
The analyses and algorithms in this paper readily accommodate non-uniform Stein sets of the form
\balign\label{eqn:non-uniform-stein-set}
	\steinset^{c_{1:3}} 
		\defeq \bigg\{ g :\xset\to\reals^d \bigg|
		\begin{array}{l}
		\sup_{x\neq y\in \xset} \maxarg{\frac{\norm{g(x)}^*}{c_1},\frac{\norm{\grad g(x)}^*}{c_2}, \frac{\norm{\grad g(x) - \grad g(y)}^*}{c_3\norm{x-y}}} \leq 1\,\text{ and}\\
		\inner{g(x)}{n(x)} = 0, \forall x \in \partial\xset \text{ with $n(x)$ defined}
		\end{array}
		\bigg\}
\ealign
for constants $c_1,c_2,c_3 > 0$ known as \emph{Stein factors} in the literature.
We will exploit this additional flexibility in \secref{comparing-discrepancies} to establish tight lower-bounding relations between the Stein discrepancy and Wasserstein distance for well-studied target distributions.
For general use, however, we advocate the parameter-free classical Stein set and graph Stein sets to be introduced in the sequel.
Indeed, any non-uniform Stein discrepancy is equivalent to the classical Stein discrepancy in a strong sense:
\begin{proposition}[Equivalence of Non-uniform Stein Discrepancies] \label{prop:stein-equiv}
For any $c_1,c_2, c_3 > 0$, 
\baligns
\min(c_1,c_2,c_3)\langstein{Q}{\steinset} \leq \langstein{Q}{\steinset^{c_{1:3}}} \leq \max(c_1,c_2,c_3)\langstein{Q}{\steinset}.
\ealigns
\end{proposition}

\section{Computing Stein Discrepancies}
\label{sec:programs}
In this section, we introduce an efficiently computable Stein discrepancy with convergence properties equivalent to those of the classical discrepancy. %
We restrict attention to the unconstrained domain $\xset=\reals^d$ in \secsref{graph-stein}-\secssref{l1-norm} and present extensions for constrained domains in \secref{constrained-programs}.
\subsection{Graph Stein Discrepancies} 
\label{sec:graph-stein}
Evaluating a Stein discrepancy $\langstein{Q}{\gset}$ for a fixed $(Q,P)$ pair reduces to solving an optimization program over functions $g\in\gset$.
For example, the classical Stein discrepancy is the optimum
\balign\label{eqn:classical-stein-program}
	\langstein{Q}{\steinset}
		=
		\sup_{g}\  &\textsum_{i=1}^n\ q(x_i) (\inner{g(x_i)}{\grad\log p(x_i)} + \inner{\grad}{g(x_i)}) \\ \notag
		\text{s.t.}\ 
		&\norm{g(x)}^* \leq 1, %
		\norm{\grad g(x)}^* \leq 1, %
		\norm{\grad g(x) - \grad g(y)}^* \leq \norm{x - y}, \forall x,y \in\xset.
\ealign
Note that the objective associated with any Stein discrepancy $\langstein{Q}{\gset}$ is linear in $g$ and, since $Q$ is discrete, 
only depends on $g$ and $\grad g$ through their values at each of the $n$ sample points $x_i$.
The primary difficulty in solving the classical Stein program \eqref{eqn:classical-stein-program} stems from the infinitude of constraints imposed by the classical Stein set $\steinset$.
One way to avoid this difficulty is to impose the classical smoothness constraints at only a finite collection of points.
To this end, for each finite graph $G = (V,E)$ with vertices $V\subset \xset$ and edges $E\subset V^2$, we define the \emph{graph Stein set},
\baligns
	&\gsteinset{}{G} 
		\defeq \bigg\{ g :\xset\to\reals^d \mid 
			\forall\, x\in V,\ \maxarg{\norm{g(x)}^*, \norm{\grad g(x)}^*} \leq 1 
			\text{  and, } \forall\, (x, y) \in E, \\
			&\maxarg{\textstyle\frac{\norm{g(x) - g(y)}^*}{\norm{x - y}},
				\textstyle\frac{\norm{\grad g(x) - \grad g(y)}^*}{\norm{x - y}},\textstyle\frac{\norm{g(x) - g(y) - {\grad g(x)}{(x - y)}}^*}{\frac{1}{2}\norm{x - y}^2},
				\textstyle\frac{\norm{g(x) - g(y) - {\grad g(y)}{(x - y)}}^*}{\frac{1}{2}\norm{x - y}^2}} \leq 1\bigg\},
\ealigns
the family of functions which satisfy the classical constraints and certain implied Taylor compatibility constraints at pairs of points in $E$.
Remarkably, if the graph $G_1$ consists of edges between all distinct sample points $x_i$, then the associated \emph{complete graph Stein discrepancy} $\langstein{Q}{\gsteinset{}{G_1}}$ is equivalent to the classical Stein discrepancy in the following strong sense. 
\begin{proposition}[Equivalence of Classical and Complete Graph Stein Discrepancies] \label{prop:complete-graph-equivalence}
If $\xset = \reals^d$, and $G_1=(\supp{Q}, E_1)$ with $E_1 = \{(x_i,x_l) \in \supp{Q}^2: x_i \neq x_l\}$, then %
\baligns
\langstein{Q}{\steinset} \leq \langstein{Q}{\gsteinset{}{G_1}} \leq \kappa_d\,\langstein{Q}{\steinset},
\ealigns
where $\kappa_d$ is a constant, independent of $(Q,P)$, depending only on the dimension $d$ and norm $\norm{\cdot}$.
\end{proposition}
\propref{complete-graph-equivalence} follows from the Whitney-Glaeser extension theorem for smooth functions~\citep{Glaeser58,Shvartsman08} and implies that the complete graph Stein discrepancy inherits all of the desirable convergence properties of the classical discrepancy.
However, the complete graph also introduces order $n^2$ constraints, rendering computation infeasible for large samples. %
To achieve the same form of equivalence while enforcing only $O(n)$ constraints, we will make use of sparse \emph{geometric spanner} subgraphs.

\subsection{Geometric Spanners}
For a given dilation factor $t \geq 1$, 
a \emph{$t$-spanner} \cite{Chew86,PelegSc89} is a graph $G = (V, E)$ with weight $\norm{x-y}$ on each edge $(x,y)\in E$
and a path between each pair $x'\neq y'\in V$ with total weight no larger than $t \norm{x' - y'}$.
The next proposition shows that \emph{spanner Stein discrepancies} enjoy the same convergence properties as the complete graph Stein discrepancy.
\begin{proposition}[Equivalence of Spanner and Complete Graph Stein Discrepancies] \label{prop:spanner-equivalence}
If $\xset = \reals^d$, $G_t = (\supp{Q}, E)$ is a $t$-spanner, and $G_1=(\supp{Q}, \{(x_i,x_l) \in \supp{Q}^2: x_i \neq x_l\})$, then
\baligns
\langstein{Q}{\gsteinset{}{G_1}} \leq \langstein{Q}{\gsteinset{}{G_t}} \leq 2t^2\,\langstein{Q}{\gsteinset{}{G_1}}.
\ealigns
\end{proposition}
Moreover, for any $\ell_p$ norm, a 2-spanner with $O(\kappa_d n)$ edges can be computed in $O(\kappa_d n\log(n))$ expected time for $\kappa_d$ a constant depending only on $d$ and $\norm{\cdot}$~\citep{Har-PeledMe06}.
As a result, we will adopt a 2-spanner Stein discrepancy, $\langstein{Q}{\gsteinset{}{G_2}}$, as our standard quality measure.

\subsection{Decoupled Linear Programs}
\label{sec:l1-norm}
The final unspecified component of our Stein discrepancy is the choice of norm $\norm{\cdot}$.
We recommend the $\ell_1$ norm, as the resulting optimization problem decouples into $d$ independent finite-dimensional linear programs (LPs) that can be solved in parallel.
More precisely, $\langstein{Q}{\gsteinset{1}{(V,E)}}$ equals
\balign\label{eqn:l1-program}
	\textsum_{j=1}^d &\sup_{{\gamma_j\in\reals^{|V|},\Gamma_j\in\reals^{d\times |V|}}}
	\ \textsum_{i=1}^{|V|}\ q(v_i) (\gamma_{ji}{\grad_j\log p(v_i)} + \Gamma_{jji}) \\ \notag
	&\ \text{s.t.}\ \	\norm{\gamma_j}_\infty \leq 1, \norm{\Gamma_j}_\infty \leq 1, \text{ and } 
	\forall\, i\neq l : (v_i, v_l) \in E, \\ \notag
	&\quad\quad\maxarg{
	\textstyle\frac{|\gamma_{ji} - \gamma_{jl}|}{\norm{v_i - v_{l}}_1}, 
	\textstyle\frac{\norm{\Gamma_{j}(e_i -  e_l)}_\infty}{\norm{v_i - v_{l}}_1}, 
	\textstyle\frac{|\gamma_{ji} - \gamma_{jl} - \inner{\Gamma_{j}e_i}{v_i - v_{l}}|}{\frac{1}{2}\norm{v_i - v_{l}}_1^2},
	\textstyle\frac{|\gamma_{ji} - \gamma_{jl} - \inner{\Gamma_{j}e_l}{v_i - v_{l}}|}{\frac{1}{2}\norm{v_i - v_{l}}_1^2}} \leq 1.
\ealign
We have arbitrarily numbered the elements $v_i$ of the vertex set $V$ so that $\gamma_{ji}$ represents the function value $g_j(v_i)$, and $\Gamma_{jki}$ represents the gradient value $\grad_k g_j(v_i)$.
\subsection{Constrained Domains}
\label{sec:constrained-programs}
A small modification to the unconstrained formulation \eqref{eqn:l1-program} extends our tractable Stein discrepancy computation to any domain defined by coordinate boundary constraints, that is, to $\xset = (\alpha_1,\beta_1)\times \dots \times(\alpha_d,\beta_d)$ with $-\infty\leq \alpha_j < \beta_j \leq \infty$ for all $j$.
Specifically, for each dimension $j$, we augment the $j$-th coordinate linear program of \eqref{eqn:l1-program} with the boundary compatibility constraints
\balign\label{eqn:graph-boundary-constraints}
\maxarg{
\textstyle\frac{|\gamma_{ji}|}{|v_{ij} - b_j|},
\textstyle\frac{|\Gamma_{jki}|}{|v_{ij} - b_j|},
\textstyle\frac{|\gamma_{ji} - \Gamma_{jji} (v_{ij}-b_j)|}{\frac{1}{2} (v_{ij}-b_j)^2}
} \leq 1,
\ \text{for each}\ i,\ b_j \in \{\alpha_j, \beta_j\}\cap\reals,\ \text{and}\ k \neq j.
\ealign
These additional constraints ensure that our candidate function and gradient values can be extended to a smooth function satisfying the boundary conditions $\inner{g(z)}{n(z)} = 0$ on $\partial\xset$.
\propref{boundary-constraints} in the appendix shows that the spanner Stein discrepancy so computed is strongly equivalent to the classical Stein discrepancy on $\xset$.

\algref{multivariate} summarizes the complete solution for computing our recommended, parameter-free spanner Stein discrepancy in the multivariate setting.
Notably, the spanner step is unnecessary in the univariate setting, as the complete graph Stein discrepancy $\langstein{Q}{\gsteinset{1}{G_1}}$ can be computed directly by sorting the sample and boundary points and only enforcing constraints between consecutive points in this ordering.
Thus, the complete graph Stein discrepancy is our recommended quality measure when $d=1$, and a recipe for its computation is given in \algref{univariate}.

\begin{algorithm}[t]
   \caption{Multivariate Spanner Stein Discrepancy} %
   \label{alg:multivariate}
\begin{algorithmic}
	\STATE {\bfseries input:} $Q$, coordinate bounds $(\alpha_1,\beta_1), \dots, (\alpha_d,\beta_d) $ with $-\infty\leq \alpha_j<\beta_j\leq\infty$ for all $j$
	\STATE $G_{2} \gets$ Compute sparse 2-spanner of $\supp{Q}$
	\STATE \textbf{for} $j=1$ \TO $d$ \textbf{do (in parallel)}
	\STATE\hspace{5mm}  $r_j \gets$ Solve $j$-th coordinate linear program \eqref{eqn:l1-program} with graph $G_{2}$ and boundary constraints \eqref{eqn:graph-boundary-constraints}
	\RETURN $\sum_{j=1}^d r_j$
\end{algorithmic}
\end{algorithm}
\begin{algorithm}[t]
   \caption{Univariate Complete Graph Stein Discrepancy} %
   \label{alg:univariate}
\begin{algorithmic}
	\STATE {\bfseries input:} $Q$, bounds $(\alpha,\beta)$ with $-\infty \leq \alpha < \beta \leq \infty$
	\STATE $(x_{(1)}, \dots, x_{(n')}) \gets \tsc{Sort}(\{x_1,\dots,x_n,\alpha,\beta\}\cap\reals)$ 
	\RETURN 
	$\sup_{\gamma\in\reals^{n'}, \,\Gamma\in\reals^{n'}} \textstyle \sum_{i=1}^{n'} q(x_{(i)}) (\gamma_{i}{\deriv{}{x}\log p(x_{(i)})} + \Gamma_{i})$
	\baligns
	&s.t.\ \norm{\Gamma}_\infty \leq 1, \forall i \leq n', |\gamma_i| \leq \indic{\alpha < x_{(i)} < \beta},\text{ and, } \forall i < n',\\
	&\maxarg{\textstyle\frac{|\gamma_{i} - \gamma_{i+1}|}{x_{(i+1)}-x_{(i)}},
	\textstyle\frac{|\Gamma_{i} - \Gamma_{i+1}|}{x_{(i+1)}-x_{(i)}},
	\textstyle\frac{|\gamma_{i} - \gamma_{i+1} - \Gamma_i(x_{(i)} - x_{(i+1)})|}{\frac{1}{2}(x_{(i+1)}-x_{(i)})^2} ,
	\textstyle\frac{|\gamma_{i} - \gamma_{i+1} - \Gamma_{i+1}(x_{(i)} - x_{(i+1)})|}{\frac{1}{2}(x_{(i+1)}-x_{(i)})^2} }\leq 1
	\ealigns
\end{algorithmic}
\end{algorithm}

\section{Experiments}
\label{sec:experiments}

We now turn to an empirical evaluation of our proposed quality measures.
We compute all spanners using the efficient C++ greedy
spanner implementation of \citet{BoutsteBu14}
and solve all optimization programs using Julia for Mathematical Programming~\citep{LubinDu15} with the default Gurobi 6.0.4 solver \cite{Gurobi15}.
All reported timings are obtained using a single core of an Intel Xeon CPU E5-2650 v2 @ 2.60GHz.
\subsection{A Simple Example}
We begin with a simple example to illuminate a few properties of the Stein diagnostic.
For the target $P = \Gsn(0,1)$, we generate a sequence of sample points i.i.d.\ from the target and a second sequence i.i.d.\ from a scaled Student's t distribution with matching variance and 10 degrees of freedom.
The left panel of \figref{julia_incorrect_target_divergence} shows that the complete graph Stein discrepancy applied to the first $n$ Gaussian sample points decays to zero at an $n^{-0.52}$ rate, while the discrepancy applied to the scaled Student's t sample remains bounded away from zero.
The middle panel displays optimal Stein functions $g$ recovered by the Stein program for different sample sizes.
Each $g$ yields a test function $h \defeq \langevin{g}$, featured in the right panel, that best discriminates the sample $Q$ from the target $P$.
Notably, the Student's t test functions exhibit relatively large magnitude values in the tails of the support.
\begin{figure}
  \centering
  \includegraphics[width=\textwidth]{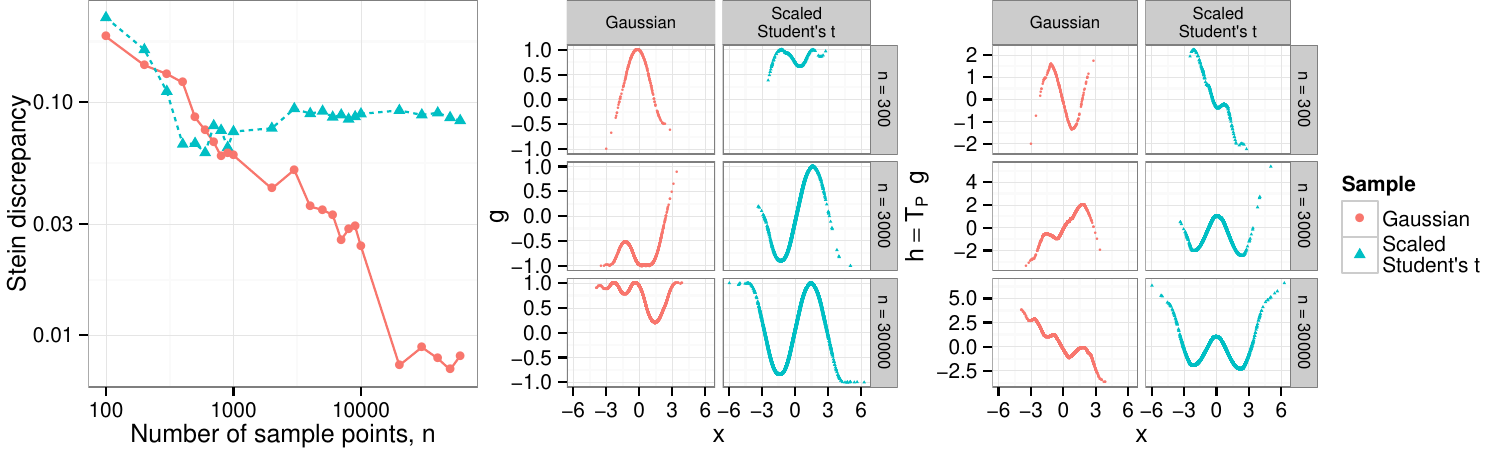}
  \caption{Left: Complete graph Stein discrepancy for a $\Gsn(0,1)$ target. 
  Middle / right: Optimal Stein functions $g$ and discriminating test
  functions $h=\langevin{g}$ recovered by the Stein program.}
  \label{fig:julia_incorrect_target_divergence}
\end{figure}

\subsection{Comparing Discrepancies} \label{sec:comparing-discrepancies}

We show in \thmref{univariate-equivalent-qcqp} in the appendix that, when
$d=1$, the classical Stein discrepancy is the optimum of a convex quadratically
constrained quadratic program with a linear objective, $O(n)$ variables, and
$O(n)$ constraints.  This offers the opportunity to directly compare the
behavior of the graph and classical Stein discrepancies.  
We will also compare to the Wasserstein distance $\wass$, which is computable for simple univariate target distributions \citep{Vallender74}
and provably lower bounds the non-uniform Stein discrepancies \eqref{eqn:non-uniform-stein-set} with $c_{1:3} = (0.5,0.5,1)$ for $P=\Unif(0,1)$ and $c_{1:3} = (1,4,2)$ for $P=\Gsn(0,1)$ \citep{ChenGoSh11,Dobler14}.
For $\Gsn(0,1)$ and $\Unif(0,1)$ targets and several random number generator seeds,
we generate a sequence of sample points i.i.d.\ from the target distribution and 
plot the non-uniform classical and complete graph Stein discrepancies and the Wasserstein distance as functions of the first $n$ sample points in \figref{julia_compare_discrepancies_d=1}.
Two apparent trends are that the graph Stein discrepancy very closely approximates the classical and that 
both Stein discrepancies track the fluctuations in Wasserstein distance even when a magnitude separation exists.
In the $\Unif(0,1)$ case, the Wasserstein distance in fact equals the classical Stein discrepancy because $\langevin{g} = g'$ is a Lipschitz function.

\subsection{Selecting Sampler Hyperparameters}
\begin{figure}
  \centering
  \includegraphics[width=1.02\textwidth]{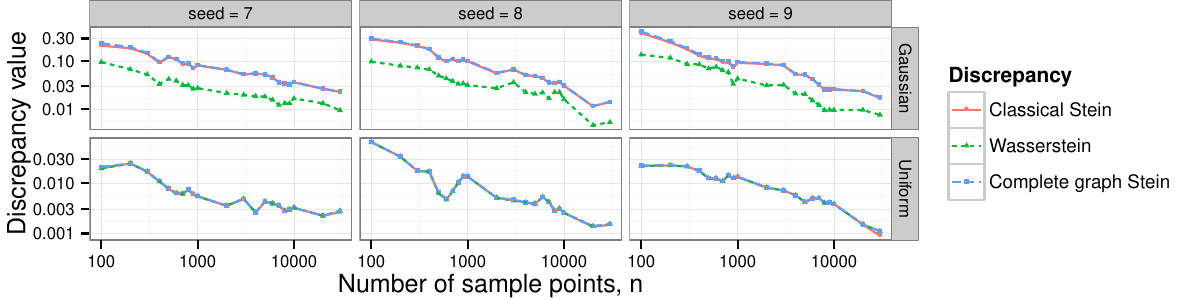}
  \caption{Comparison of discrepancy measures for sample sequences drawn i.i.d.\ from their targets.}
  \label{fig:julia_compare_discrepancies_d=1}
\end{figure}
Stochastic Gradient Langevin Dynamics (SGLD) \cite{WellingTe11} with constant step size $\eps$ is a biased MCMC procedure designed for scalable inference.
It approximates the overdamped Langevin diffusion, but, because no Metropolis-Hastings (MH) correction is used, 
the stationary distribution of SGLD deviates increasingly from its target as $\eps$ grows.
If $\epsilon$ is too small, however, SGLD explores the sample space too slowly.
Hence, an appropriate choice of $\epsilon$ is critical for accurate posterior inference.
To illustrate the value of the Stein diagnostic for this task, 
we adopt the bimodal Gaussian mixture model (GMM) posterior of \citep{WellingTe11} as our target.
For a range of step sizes $\eps$, we use SGLD with minibatch size 5 to draw 50 independent sequences of length $n = 1000$,
and we select the value of $\eps$ with the highest median quality --  either the maximum effective sample size (ESS, a standard diagnostic based on autocorrelation \citep{BrooksGeJoMe11}) or the minimum spanner Stein discrepancy -- across these sequences.
The average discrepancy computation consumes $0.4$s for spanner construction and $1.4$s per coordinate linear program.
As seen in \figref{julia_sgld-gmm_diagnostic_n=1000}, ESS, which does not detect distributional bias, selects the largest step size presented to it, while
the Stein discrepancy prefers an intermediate value.
The rightmost plot of \figref{julia_diagnostic_contour_distname=sgld-gmm_n=1000_seed=7} 
shows that a representative SGLD sample of size $n$ using the $\eps$ selected by ESS is greatly overdispersed; the leftmost is greatly underdispersed due to slow mixing.
The middle sample, with $\eps$ selected by the Stein diagnostic, most closely resembles the true posterior.

\subsection{Quantifying a Bias-Variance Trade-off}
The approximate random walk MH (ARWMH) sampler \cite{Korattikara2014} is a second biased MCMC procedure designed for scalable posterior inference.
Its tolerance parameter $\eps$ controls the number of datapoint likelihood evaluations used to approximate the standard MH correction step.
Qualitatively, a larger $\eps$ implies fewer likelihood computations, more rapid sampling, and a more rapid reduction of variance.
A smaller $\eps$ yields a closer approximation to the MH correction and less bias in the sampler stationary distribution.
We will use the Stein discrepancy to explicitly quantify this bias-variance trade-off. 

We analyze a dataset of $53$ prostate cancer patients with six binary predictors and a binary outcome indicating whether cancer has spread to surrounding lymph nodes~\citep{CantyRi15}.
Our target is the Bayesian logistic regression posterior~\cite{BrooksGeJoMe11} under a $\Gsn(0,I)$ prior on the parameters.
We run RWMH ($\eps = 0$) and ARWMH ($\eps = 0.1$ and batch size $= 2$)
for $10^5$ likelihood evaluations,
discard the points from the first $10^3$ evaluations,
and thin the remaining points to sequences of length $1000$.
The discrepancy computation time for $1000$ points averages $1.3$s for the spanner and $12$s for a coordinate LP.
\figref{approxmh_nodal} displays the spanner Stein discrepancy applied to the first $n$ points in each sequence as a function of the likelihood evaluation count.
We see that the approximate sample is of higher Stein quality for smaller computational budgets but is eventually overtaken by the asymptotically exact sequence.

To corroborate our result, we use a Metropolis-adjusted Langevin chain~\citep{RobertsTw96} of length $10^7$ as a surrogate $Q^*$ for the target 
and compute several error measures for each sample $Q$:
normalized probability error ${\max_l{ |\Earg{\sigma(\inner{X}{w_l})-\sigma(\inner{\PVAR}{w_l})}|}}{/\infnorm{w_l}}$,
mean error $\frac{\max_j |\Earg{X_j-\PVAR_j}|}{\max_j |\Esubarg{Q^*}{\PVAR_j}|}$,
and 
second moment error $\frac{\max_{j,k} |\Earg{X_jX_k-\PVAR_j\PVAR_k}|}{\max_{j,k} |\Esubarg{Q^*}{\PVAR_j\PVAR_k}|}$
for $X\sim Q$, $Z\sim Q^*$, $\sigma(t) \defeq\frac{1}{1+e^{-t}}$, and $w_l$ the $l$-th datapoint covariate vector.
The measures, also found in \figref{approxmh_nodal}, accord with the Stein discrepancy quantification.
\subsection{Assessing Convergence Rates}
\begin{figure}
  \centering
  \begin{subfigure}[b]{0.3\textwidth}
    \includegraphics[width=\textwidth]{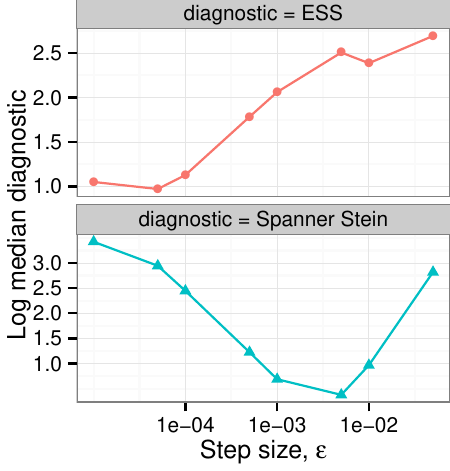}
    \caption{Step size selection criteria}
    \label{fig:julia_sgld-gmm_diagnostic_n=1000}
  \end{subfigure}
  \quad
  \begin{subfigure}[b]{0.63\textwidth}
    \includegraphics[width=\textwidth]{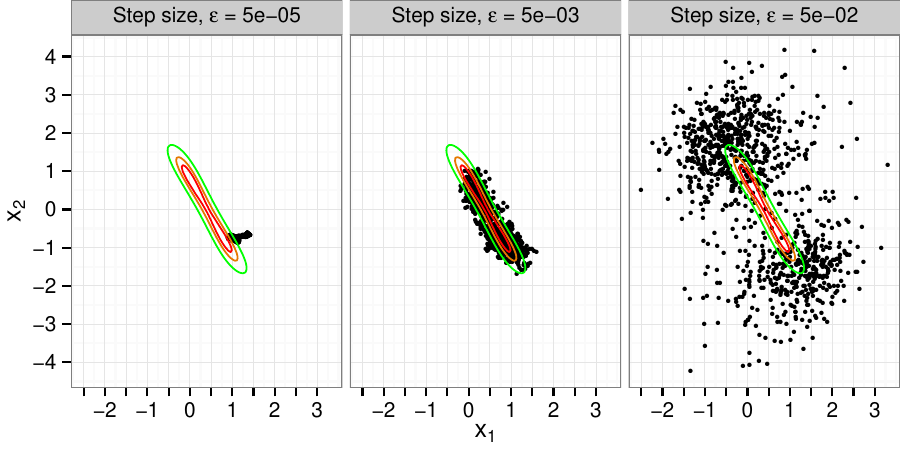}
    \caption{1000 SGLD sample points with equidensity contours of $p$ overlaid}
    \label{fig:julia_diagnostic_contour_distname=sgld-gmm_n=1000_seed=7}
  \end{subfigure}
  \caption{
    (\subref{fig:julia_sgld-gmm_diagnostic_n=1000}) 
      ESS maximized at $\epsilon=5\times 10^{-2}$;
      Stein discrepancy minimized at $\epsilon = 5\times
      10^{-3}$.\\
  }
\end{figure}
\begin{figure}
  \centering
  \includegraphics[width=0.95\textwidth]{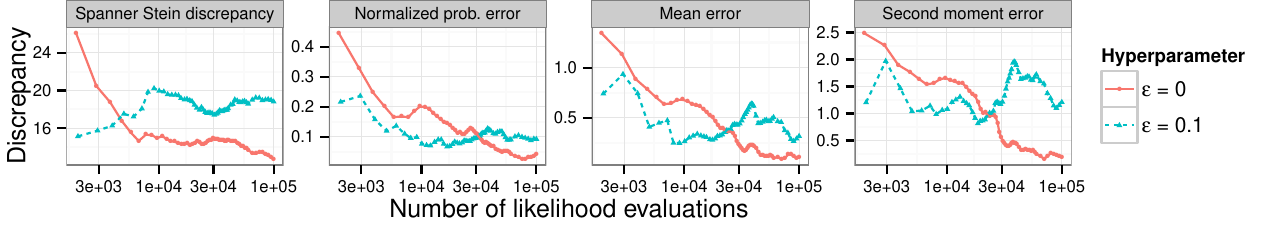}
  \caption{Bias-variance trade-off curves for Bayesian logistic regression with approximate RWMH.}
  \label{fig:approxmh_nodal}
\end{figure}
The Stein discrepancy can also be used to assess the quality of deterministic sample sequences.
In \figref{pseudosample_assessing_convergence_rates} in the appendix,
for $P = \Unif(0,1)$, we plot the complete graph Stein discrepancies of the first $n$ points of an i.i.d.\ $\Unif(0,1)$ sample, 
a deterministic Sobol sequence \cite{Sobol67}, 
and a deterministic kernel herding sequence \cite{Chen2010} defined by the norm $\norm{h}_{\hset}=\int_0^1 (h'(x))^2dx$.
We use the median value over 50 sequences in the i.i.d.\ case and estimate the convergence rate for each sampler using the slope of the best least squares affine fit to each log-log plot.
The discrepancy computation time averages $0.08$s for $n=200$ points, and the recovered rates of $n^{-0.49}$ and $n^{-1}$ for the i.i.d.\ and Sobol sequences accord with expected $O(1/\sqrt{n})$ and $O(1/n)$ bounds from the literature \cite{delBarrioGiMa99,WangSl08}.
As witnessed also in other metrics \cite{Bach2012},  the herding rate of $n^{-0.96}$ outpaces its best known bound of $\ipm(Q_n,P)=O(1/\sqrt{n})$, suggesting an opportunity for sharper  analysis.

\section{Discussion of Related Work}
\label{sec:conclusions}
We have developed a quality measure suitable for comparing biased, exact, and deterministic sample sequences by exploiting an infinite class of known target functionals.
The diagnostics of \citep{ZellnerMi95,FanBrGe06} also account for asymptotic bias but lose discriminating power by considering only a finite collection of functionals.  For example, for a $\Gsn(0,1)$ target, the score statistic of \citep{FanBrGe06} cannot distinguish two samples with equal first and second moments.
Maximum mean discrepancy (MMD) on a characteristic Hilbert space \citep{GrettonBoRaScSm06} takes full distributional bias into account but is only viable when the expected kernel evaluations are easily computed under the target.
One can approximate MMD, but this requires access to a separate trustworthy ground-truth sample from the target.

\setlength{\bibsep}{0pt plus 0.3ex}
\bibliographystyle{unsrtnat}
\ifdefined\arxivmode
\else
\subsubsection*{Acknowledgments}
The authors thank Madeleine Udell, Andreas Eberle, and Jessica Hwang for their pointers and feedback and Quirijn Bouts, Kevin Buchin, and Francis Bach for sharing their code and counsel.
{\small

}
\fi
\appendix
\begin{figure}[th]
  \centering
  \includegraphics[width=\textwidth]{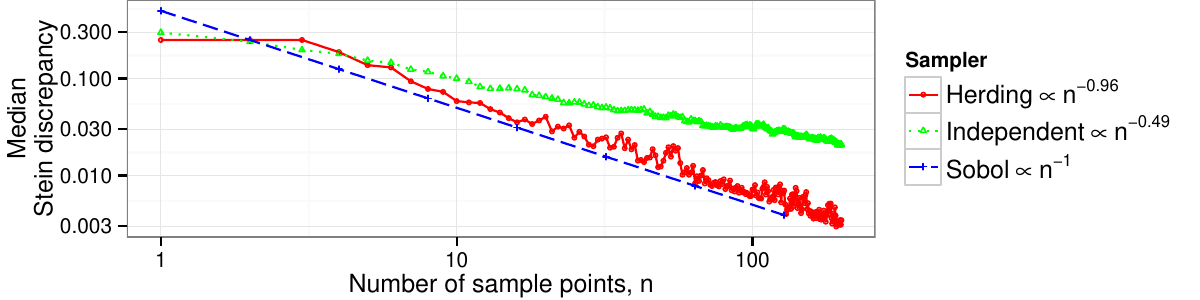}
  \caption{Comparison of complete graph Stein discrepancy convergence for $P=\Unif(0,1)$.}
  \label{fig:pseudosample_assessing_convergence_rates}
\end{figure}
\section{Proof of {\propref{langevin-zero}}}
Our integrability assumption together with the boundedness of $g$ and $\grad g$ imply that $\Esubarg{P}{\inner{\grad}{g(\PVAR)}}$
and $\Esubarg{P}{\inner{g(\PVAR)}{\grad \log p(\PVAR)}}$ exist.
Define the $\ell_\infty$ ball of radius $r$, $\ball_r = \{x\in\reals^d : \norm{x}_\infty\leq r\}$.
Since $\xset$ is convex, the intersection $\xset \cap \ball_r$ is compact and convex with Lipschitz boundary $\partial(\xset \cap \ball_r)$.
Thus, the divergence theorem (integration by parts) implies that
\baligns
	\Esubarg{P}{\langarg{g}{\PVAR}}
		&= \Esubarg{P}{\inner{\grad}{g(\PVAR)}+\inner{g(\PVAR)}{\grad \log p(\PVAR)}} 
		= \int_{\xset} \inner{\grad}{p(\pvar)g(\pvar)}\, d\pvar \\
		&= \lim_{r \to\infty}\int_{\xset \cap \ball_r} \inner{\grad}{p(\pvar)g(\pvar)}\, d\pvar 
		=  \lim_{r \to\infty} \int_{\partial(\xset \cap \ball_r)} \inner{g(\pvar)}{n_r(\pvar)}p(\pvar)\, d\pvar
\ealigns
for $n_r$ the outward unit normal vector to $\partial(\xset \cap \ball_r)$.
The final quantity in this expression equates to zero, as $\inner{g(x)}{n(x)} = 0$ for all $x$ on the boundary $\partial\xset$, 
$g$ is bounded, and $\lim_{m\to\infty}p(x_m) = 0$ for any $(x_m)_{m=1}^\infty$ with $x_m \in \xset$ for all $m$ and $\norm{x_m}_\infty\to\infty$.

\section{Proof of {\thmref{concave-stein-lower-bound}}: Stein Discrepancy Lower Bound for Strongly Log-concave Densities}
We let $C^k(\xset)$ denote the set of real-valued functions on $\xset$ with $k$ continuous derivatives
and $\smooth$ denote the \emph{smooth function distance}, the IPM generated by 
\[
	\textstyle\smoothset 
		\defeq \left\{h \in C^3(\xset) \smiddle
			\sup_{x\in\xset}\maxarg{\norm{\grad h(x)}^*, 
				\norm{\Hess h(x)}^*,
				\norm{\grad^3 h(x)}^*} \leq 1 \right\}.
\]
We additionally define the operator norms $\opnorm{v} \defeq \twonorm{v}$ for vectors $v\in \reals^d$, $\opnorm{M} \defeq \sup_{v\in \reals^d: \twonorm{v}=1}{\twonorm{Mv}}$ for matrices $M\in \reals^{d \times d}$ ,
and $\opnorm{T} \defeq \sup_{v\in \reals^d: \twonorm{v}=1}{\opnorm{T[v]}}$ for tensors $T \in \reals^{d\times d\times d}$.

The following result, %
proved in the companion paper~\citep{MackeyGo15}, 
establishes the existence of explicit constants (\emph{Stein factors}) $c_1, c_2, c_3>0$, such that, for any test function $h\in\smoothset$, the \emph{Stein equation}
\[
	h(x) - \Esubarg{P}{h(Z)} 
		= \langarg{g_h}{x}
\]
has a solution $g_h=\frac{1}{2}\grad u_h$ belonging to the non-uniform Stein set $\steinset^{c_{1:3}}$.
\begin{theorem}[Stein Factors for Strongly Log-concave Densities~{\citep[Theorem 2.1]{MackeyGo15}}] \label{thm:concave-stein-factors}
Suppose that $\xset = \reals^d$ and that $\log p \in C^4(\xset)$ is $k$-strongly concave with 
\[\sup_{z\in\xset}\opnorm{\grad^3 \log p(z)} \leq L_3 \qtext{and} \sup_{z\in\xset}\opnorm{\grad^4 \log p(z)} \leq L_4.\]
For each $x \in\xset$, let $(\process{t}{x})_{t\geq0}$ represent the overdamped Langevin diffusion with infinitestimal generator
\balign\label{eqn:langevin-generator}
\genarg{u}{x} = \frac{1}{2}\inner{\grad u(x)}{\grad\log p(x)} + \frac{1}{2}\inner{\grad}{\grad u(x)}
\ealign and initial state $\process{0}{x} = x$.
Then, for each $h \in C^3(\xset)$ with bounded first, second, and third derivatives, the function
\[
	u_h(x) \defeq \int_0^\infty \Esubarg{P}{h(\PVAR)} - \Earg{h(\process{t}{x})} \ dt
\]
solves the the Stein equation
\balign\label{eqn:stein-equation-generator}
h(x) - \Esubarg{P}{h(Z)} 
	= \genarg{u_h}{x}
\ealign
and satisfies
\baligns
	\sup_{z \in \xset} \twonorm{\grad u_h(z)} 
		\leq\ &\frac{2}{k} \sup_{z \in \xset} \twonorm{\grad h(z)}, \\
	\sup_{z \in \xset} \opnorm{\Hess u_h(z)} 
		\leq\ &\frac{2L_3}{k^2}\sup_{z \in \xset} \twonorm{\grad h(z)}+\frac{1}{k}\sup_{z \in \xset} \opnorm{\Hess h(z)},\ and \\
	\sup_{z,y \in \xset, z\neq y}\frac{\opnorm{\Hess u_h(z) - \Hess u_h(y)}}{\twonorm{z-y}} 
		\leq\ &\frac{6L_3^2}{k^3}\sup_{z \in \xset} \twonorm{\grad h(z)}+\frac{L_4}{k^2}\sup_{z \in \xset} \twonorm{\grad h(z)}\\
		+\ &\frac{3L_3}{k^2}\sup_{z \in \xset} \opnorm{\Hess h(z)} + \frac{2}{3k}\sup_{z \in \xset} \opnorm{\grad^3 h(z)}.
\ealigns
\end{theorem}
Hence, by the equivalence of non-uniform Stein discrepancies (\propref{stein-equiv}),
$\smooth(\mu, P) \leq \langstein{\mu}{\steinset^{c_{1:3}}} \leq \max(c_1,c_2,c_3) \langstein{\mu}{\steinset}$ for any probability measure $\mu$.

The desired result now follows from \lemref{smooth-wass}, which implies that the Wasserstein distance $\wass(\mu_m,P) \to 0$ whenever $\smooth(\mu_m, P) \to 0$ for a sequence of probability measures $(\mu_m)_{m\geq 1}$.
\begin{lemma}[Smooth-Wasserstein Inequality]\label{lem:smooth-wass}
If $\mu$ and $\nu$ are probability measures on $\reals^d$, and $\norm{v} \geq \norm{v}_2$ for all $v\in\reals^d$, then 
\baligns
\smooth(\mu,\nu) \leq \wass(\mu,\nu) \leq 3\maxarg{\smooth(\mu, \nu), \sqrt[3]{\smooth(\mu, \nu)\sqrt{2}\,\Earg{\norm{G}}^{2}}}.
\ealigns
for $G$ a standard normal random vector in $\reals^d$.
\end{lemma}
Lemma 2.2 of the companion paper~\citep{MackeyGo15} establishes this result for the case $\norm{\cdot} = \twonorm{\cdot}$; we omit the proof of the generalization which closely mirrors that of the Euclidean norm case.
\section{Proof of {\propref{stein-upper-bound}}: Stein Discrepancy Upper Bound}
Fix any $g$ in $\steinset$. 
By \propref{langevin-zero}, $\Earg{\langarg{g}{\PVAR}} = 0$.
The Lipschitz and boundedness constraints on $g$ and $\grad g$ now yield
\baligns
\Esubarg{Q}{\langarg{g}{X}} 
	&= \Earg{\langarg{g}{X} - \langarg{g}{\PVAR}}  \\
	&= \Earg{\inner{g(X)}{\grad\log p(X)} - \inner{g(\PVAR)}{\grad\log p(\PVAR)}+ \inner{\grad}{g(X)-g(\PVAR)}}\\
	&= \Earg{\inner{g(X)}{\grad\log p(X)-\grad\log p(\PVAR)} 
	+ \inner{g(X)-g(\PVAR)}{\grad\log p(\PVAR)}} \\
	&+ \Earg{\inner{\grad}{g(X)-g(\PVAR)}} \\
	&\leq \Earg{\norm{\grad\log p(X)-\grad\log p(\PVAR)}} +
\Earg{\norm{\grad \log p(\PVAR)(X-\PVAR)^\top}} + \norm{I}\Earg{\norm{X-\PVAR}}.
\ealigns

To derive the second advertised inequality, we use the definition of the matrix norm, the Fenchel-Young inequality for dual norms,
the definition of the matrix dual norm, and the Cauchy-Schwarz inequality in turn:
\baligns
\Earg{\norm{\grad \log p(\PVAR)(X-\PVAR)^\top}} 
	&= \Earg{\sup_{M:\norm{M}^* = 1} \inner{\grad \log p(\PVAR)}{M(X-\PVAR)}} \\
	&\leq \Earg{\sup_{M:\norm{M}^* = 1}\norm{\grad \log p(\PVAR)} \norm{M(X-\PVAR)}^*} \\
	&\leq \Earg{\norm{\grad \log p(\PVAR)} \norm{X-\PVAR}}
	\leq\ \textstyle\sqrt{\Earg{{\norm{\grad \log p(Z)}}^2}\Earg{\norm{X-Z}^2}}.
\ealigns
Since our bounds hold uniformly for all $g$ in $\steinset$, the proof is complete.

\section{Proof of {\propref{stein-equiv}}: Equivalence of Non-uniform Stein Discrepancies}
Fix any $c_1, c_2, c_3 >0$, and let $c_{\max} = \max(c_1,c_2,c_3)$ and $c_{\min} = \min(c_1,c_2, c_3)$.
Since the Stein discrepancy objective is linear in $g$, we have $a\,\langstein{Q}{\steinset} = \langstein{Q}{a\steinset}$ for any $a > 0$.
The result now follows from the observation that $c_{\min}\steinset \subseteq \steinset^{c_{1:3}} \subseteq c_{\max}\steinset$.
\section{Proof of {\propref{complete-graph-equivalence}}: Equivalence of Classical and Complete Graph Stein Discrepancies}
The first inequality follows from the fact that $\steinset \subseteq \gsteinset{}{G_1}$.
By the Whitney-Glaeser extension theorem~\cite[Thm. 1.4]{Shvartsman08} of \citet{Glaeser58}, for every function
$g \in \gsteinset{}{G_1}$, there exists a function $\tilde{g} \in \kappa_d\,\steinset^*$ with $g(x_i) = \tilde{g}(x_i)$ and $\grad g(x_i) = \grad\tilde{g}(x_i)$ for all $x_i$ in the support of $Q$.
Here $\kappa_d$ is a constant, independent of $(Q,P)$, depending only on the dimension $d$ and norm $\norm{\cdot}$.
Since the Stein discrepancy objective is linear in $g$ and depends on $g$ only through the values $g(x_i)$ and $\grad g(x_i)$, we have
$\langstein{Q}{\gsteinset{}{G_1}} \leq \langstein{Q}{\kappa_d\steinset} =\kappa_d  \,\langstein{Q}{\steinset}$.

\section{Proof of {\propref{spanner-equivalence}}: Equivalence of Spanner and Complete Graph Stein Discrepancies}
\label{sec:spanner-equivalence-proof}
The first inequality follows from the fact that $\gsteinset{}{G_1} \subseteq \gsteinset{}{G_t}$.
Fix any $g \in \gsteinset{}{G_t}$ and any pair of points $z, z' \in \supp{Q}$.
By the definition of $\gsteinset{}{G_t}$, we have $\maxarg{\norm{g(z)}^*, \norm{\grad g(z)}^*} \leq 1$.
By the $t$-spanner property, there exists a sequence of points $z_0, z_1, z_2, \dots, z_{L-1}, z_{L}\in \supp{Q}$ with $z_0=z$ and $z_L=z'$ for which $(z_{l-1}, z_{l})\in E$ for all $1\leq l \leq L$ and 
$\sum_{l=1}^L \norm{z_{l-1} - z_{l}} \leq t\norm{z_0-z_L}$.
Since $\maxarg{\frac{\norm{g(z_{l-1}) - g(z_{l})}^*}{\norm{z_{l-1}-z_l}} ,\frac{\norm{\grad g(z_{l-1}) - \grad g(z_{l})}^*}{\norm{z_{l-1}-z_l}}} \leq 1$ for each $l$, the triangle inequality implies that
\[
	\norm{\grad g(z_{0}) - \grad g(z_{L})}^*
	\leq \sum_{l=1}^L \norm{\grad g(z_{l-1}) - \grad g(z_{l})}^*
	\leq \sum_{l=1}^L \norm{z_{l-1} - z_{l}} \leq t\norm{z_0 - z_L}.
\]
Identical reasoning establishes that $\norm{g(z_{0}) - g(z_{L})}^* \leq t\norm{z_0 - z_L}$.

Furthermore, since $\norm{g(z_{l-1}) - g(z_{l}) - {\grad g(z_l)}{(z_{l-1} - z_{l})}}^* \leq \frac{1}{2}\norm{z_{l-1}-z_l}^2$ for each $l$, the triangle inequality and the definition of the tensor norm $\norm{\cdot}^*$ imply that
\baligns
	&\norm{g(z_0) - g(z_L) - {\grad g(z_L)}{(z_0 - z_L)}}^*\\
		&\leq \sum_{l=1}^L \norm{g(z_{l-1}) - g(z_{l}) - {\grad g(z_l)}{(z_{l-1} - z_{l})}}^* +
		\norm{{(\grad g(z_l) - \grad g(z_L))}{(z_{l-1} - z_{l})}}^* \\
		&\leq \sum_{l=1}^L \frac{1}{2}\norm{z_{l-1} - z_{l}}^2 + \norm{\grad g(z_l) - \grad g(z_L)}^*\norm{z_{l-1} - z_{l}} \\
		&\leq \sum_{l=1}^L \frac{1}{2}\norm{z_{l-1} - z_{l}}^2 + \sum_{l'=l}^{L-1}\norm{\grad g(z_{l'}) - \grad g(z_{l'+1})}^*\norm{z_{l-1} - z_{l}} \\
		&\leq \sum_{l=1}^L \norm{z_{l-1} - z_{l}}\left(\frac{1}{2}\norm{z_{l-1} - z_{l}}+\sum_{l'=l}^{L-1}\norm{z_{l'} - z_{l'+1}}\right)
		\leq \left(\sum_{l=1}^L \norm{z_{l-1} - z_{l}}\right)^2
		\leq t^2\norm{z_0 - z_L}^2.
\ealigns
Since $z, z'$ were arbitrary, and the Stein discrepancy objective is linear in $g$, we conclude that
$\langstein{Q}{\gsteinset{}{G_t}} \leq \langstein{Q}{2t^2\gsteinset{}{G_1}} =2t^2  \,\langstein{Q}{\gsteinset{}{G_1}}$.
\section{Finite-dimensional Classical Stein Program}
\begin{theorem}[Finite-dimensional Classical Stein Program]
\label{thm:univariate-equivalent-qcqp}
If $\xset=(\alpha,\beta)$ for $-\infty \leq \alpha < \beta \leq \infty$, and $x_{(1)} < \dots < x_{(n')}$ represent the sorted values of $\{x_1,\dots,x_n,\alpha,\beta\}\cap\reals$, then 
the non-uniform classical Stein discrepancy
$\langstein{Q}{\steinset^{c_{1:3}}}$ is the optimal value of the
convex program
\begin{subequations}
\label{eqn:stein_univariate_program}
\begin{align}
  \underset{g}{\max}\quad
  &\textsum_{i=1}^{n'}\textstyle q(x_{(i)}) \deriv{}{x} \log p(x_{(i)}) g(x_{(i)}) + q(x_{(i)}) g'(x_{(i)}) & \\
  {\text{s.t.}}\quad &\forall i\in\{1,\dots,n'-1\},\ 
  |g'(x_{(i)})| \le c_2,\ 
  |g(x_{(i+1)}) - g(x_{(i)})| \le c_2 (x_{(i+1)}-x_{(i)}), & \label{eqn:qcqp_c_constraints} \\
  &g(x_{(i)}) - g(x_{(i+1)}) + \frac{1}{4c_3}\left (g'(x_{(i)}) - g'(x_{(i+1)}) \right)^2 +
  \frac{x_{(i+1)}-x_{(i)}}{2}\left (g'(x_{(i)}) + g'(x_{(i+1)}) \right)\nonumber \\
  &\quad+ \frac{1}{c_3}(L_b)_+^2\le \frac{c_3}{4}(x_{(i+1)}-x_{(i)})^2, & \label{eqn:qcqp_lb} \\
  & g(x_{(i+1)}) - g(x_{(i)}) + \frac{1}{4c_3}\left (g'(x_{(i)}) -
  g'(x_{(i+1)})\right )^2 - \frac{x_{(i+1)}-x_{(i)}}{2}\left
  (g'(x_{(i)}) + g'(x_{(i+1)})\right ) \nonumber \\
  &\quad+ \frac{1}{c_3}(L_u)_+^2 \le \frac{c_3}{4}(x_{(i+1)}-x_{(i)})^2, \qtext{and} & \label{eqn:qcqp_lu} \\
  &\forall i\in\{1,\dots,n'\}, |g(x_{(i)})| \le \indic{\alpha < x_{(i)} < \beta} (c_1 - \frac{1}{2c_3}g'(x_{(i)})^2) & \label{eqn:qcqp_buffer_constraints}
\end{align}
\end{subequations}
where $(r)_+ \defeq \maxarg{r,0}$, 
\baligns
&\textstyle L_b \defeq \frac{c_3}{2} (x_{(i+1)}-x_{(i)}) - \frac{1}{2}\left (g'(x_{(i)}) + g'(x_{(i+1)})\right) - c_2, \qtext{and} \\
&\textstyle L_u \defeq \frac{c_3}{2} (x_{(i+1)}-x_{(i)}) + \frac{1}{2}\left (g'(x_{(i)}) + g'(x_{(i+1)}) \right ) -c_2.
\ealigns
\end{theorem}
We say the program \eqref{eqn:stein_univariate_program} is finite-dimensional, because it suffices to optimize over vectors $\gamma, \Gamma \in \reals^{n'}$
representing the function values ($\gamma_i = g(x_{(i)})$) and derivative values ($\Gamma_i = g'(x_{(i)})$) at each sample or boundary point $x_{(i)}$.
Indeed, by introducing slack variables, this program is representable as a convex quadratically constrained quadratic program with $O(n)$ constraints, $O(n)$ variables, and a linear objective.
Moreover, the pairwise constraints in this program are only enforced between
neighboring points in the sequence of ordered locations $x_{(i)}$. Hence the resulting constraint matrix
is sparse and banded, making the problem particularly amenable to efficient
optimization.

\begin{proof}
Throughout, we say that $\tilde{g}$ is an {extension} of $g$ if $\tilde{g}(x_{(i)})=g(x_{(i)})$ and $\tilde{g}'(x_{(i)}) = g'(x_{(i)})$ for each $x_{(i)}\in\supp{Q}$.
Since the Stein objective only depends on $g$ and $g'$ through their values at sample points, $g$ and any extension $\tilde{g}$ have identical objective values.

We will establish our result by showing that every $g\in\steinset^{c_{1:3}}$ is feasible for the program (\ref{eqn:stein_univariate_program}),
so that $\langstein{Q}{\steinset^{c_{1:3}}}$ lower bounds the optimum of \eqref{eqn:stein_univariate_program},
and that every feasible $g$ for (\ref{eqn:stein_univariate_program}) has an extension in $\tilde{g}\in\steinset^{c_{1:3}}$, so that $\langstein{Q}{\steinset^{c_{1:3}}}$ also upper bounds the optimum of \eqref{eqn:stein_univariate_program}.

\subsection{Feasibility of $\steinset^{c_{1:3}}$}
Fix any $g\in\steinset^{c_{1:3}}$.
Also, since $g'$ is $c_2$-bounded and $c_3$-Lipschitz, the constraints \eqref{eqn:qcqp_c_constraints} must be satisfied.
Consider now the $c_2$-bounded and $c_3$-Lipschitz extensions of $g'$
\baligns
B(t) &\triangleq \max(-c_2, \max_{1\leq i\leq n'} \left [g'(x_{(i)}) -
  c_3|t - x_{(i)}|\right ])\,\text{ and }\, \\
U(t) &\triangleq \min(c_2, \min_{1\leq i\leq n'} \left [g'(x_{(i)}) +
  c_3|t - x_{(i)}|\right ]).
\ealigns
We know that $B(t) \le g'(t) \le U(t)$ for all $t$, for, if not, there would
be a point $t_0$ and a point $x_{(i)}$ such that $|g'(x_{(i)}) - g'(t_0)| > c_3 |x_{(i)}-
t_0|$, which combined with the $c_3$-Lipschitz property would be a
contradiction.
Thus, for each sample $x_{(i)}$, the fundamental theorem of calculus gives
\baligns
g(x_{(i+1)}) - g(x_{(i)}) = \int_{x_{(i)}}^{x_{(i+1)}} g'(t)\,d t \ge \int_{x_{(i)}}^{x_{(i+1)}} B(t)\,d t.
\ealigns
The right-hand side of this inequality evaluates precisely to the right-hand side of the constraint \eqref{eqn:qcqp_lb}.
An analogous upper bound using $U(t)$ yields \eqref{eqn:qcqp_lu}.

Finally, consider any point $x_{(i)}$.
If $x_{(i)} \in \{\alpha,\beta\}$, then (\ref{eqn:qcqp_buffer_constraints}) is satisfied as $g(z) = 0$ for any point $z$ on the boundary.
Suppose instead that $\alpha < x_{(i)} < \beta$. 
Without loss of generality, we may assume that $g'(x_{(i)}) \ge 0$.
Since $g'$ is $c_3$-Lipschitz, we have $g'(t) \geq g'(x_{(i)}) - c_3 |t-x_{(i)}|$ for all $t$.
Integrating both sides of this inequality from $x_{(i)}$ to $x_u = x_{(i)} + g'(x_{(i)})/c_3$, we obtain
\[
g(x_u) - g(x_{(i)}) = \int_{x_{(i)}}^{x_u} g'(t)\ dt \geq \int_{x_{(i)}}^{x_u} g'(x_{(i)}) - c_3 (t-x_{(i)})\ dt = g'(x_{(i)})^2/(2c_3)
\]
Since $g(x_u) \leq c_1$, we have
$\frac{1}{2c_3}g'(x_{(i)})^2 + g(x_{(i)}) \le c_1$. Similarly, by integrating the
inequality from $x_b = x_{(i)} - g'(x_{(i)})/c_3$ to $x_{(i)}$, we have $g(x_b) - g(x_{(i)})
\ge g'(x_{(i)})^2/(2c_3)$, which combined with $g(x_b) \leq c_1$ yields
(\ref{eqn:qcqp_buffer_constraints}).

\subsection{Extending Feasible Solutions}
Suppose now that $g$ is any function feasible for the program
(\ref{eqn:stein_univariate_program}).  We will construct an extension
$\tilde{g}\in\steinset^{c_{1:3}}$ by first working independently over each
interval $(x_{(i)},x_{(i+1)})$.  Fix an index $i < n'$.  Our strategy is to
identify a pair of $c_2$-bounded, $c_3$-Lipschitz functions $m_i$ and $M_i$
defined on the interval $[x_{(i)}, x_{(i+1)}]$ which satisfy $m_i(x)\le M_i(x)$ for
all $x\in [x_{(i)}, x_{(i+1)}]$, $m_i(x)=M_i(x)=g'(x)$ for $x\in\{x_{(i)}, x_{(i+1)}\}$,
and $\int_{x_{(i)}}^{x_{(i+1)}} m_i(t)d t \le g(x_{(i+1)})-g(x_{(i)}) \le
\int_{x_{(i)}}^{x_{(i+1)}}M_i(t)d t$.  For any such $(m_i,M_i)$ pair, there exists
$\zeta_i\in [0,1]$ satisfying
\begin{equation*}
g(x_{(i+1)}) - g(x_{(i)}) = \int_{x_{(i)}}^{x_{(i+1)}} \zeta_i m_i(t) + (1-\zeta_i)M_i(t)d t,
\end{equation*}
and hence we will define the extension
\begin{equation*}
\tilde{g}(x) = g(x_{(i)}) + \int_{x_{(i)}}^x \zeta_i m_i(t) + (1-\zeta_i)M_i(t)d t.
\end{equation*}
By convexity, the extension derivative $\tilde{g}'$ is $c_2$-bounded and
$c_3$-Lipschitz, so we will only need to check that $\sup_{x\in\xset}
|\tilde{g}(x)| \leq c_1$.  The maximum magnitude values of $\tilde{g}$ occur
either at the interval endpoints, which are $c_1$-bounded by
(\ref{eqn:qcqp_buffer_constraints}), or at critical points $x$ satisfying
$\tilde{g}'(x) = 0$, so it suffices to ensure that $\tilde{g}$ is $c_1$-bounded at all critical points.

We will use the $c_2$-bounded, $c_3$-Lipschitz functions $B$ and $U$ as building blocks for our extension,
since they satisfy $B(t)=U(t)=g'(t)$ for $t\in\{x_{(i)}, x_{(i+1)}\}$ and $B(t)\leq g'(t) \leq U(t)$, 
\baligns
B(t) &= \max(-c_2, g'(x_{(i)}) - c_3(t - x_{(i)}), g'(x_{(i+1)}) - c_3(x_{(i+1)} - t)), \qtext{and} \\
U(t) &= \min(c_2, g'(x_{(i)}) + c_3(t - x_{(i)}), g'(x_{(i+1)}) + c_3(x_{(i+1)} - t)),
\ealigns
for $t\in[x_{(i)}, x_{(i+1)}]$.
We need only consider three cases.
\paragraph{Case 1: $B$ and $U$ are never negative or never positive on $[x_{(i)}, x_{(i+1)}]$.}
For this case, we will choose $m_i = B$ and $M_i = U$.
By (\ref{eqn:qcqp_lb}) and (\ref{eqn:qcqp_lu}) we know $\int_{x_{(i)}}^{x_{(i+1)}}
m_i(t)d t \le g(x_{(i+1)})-g(x_{(i)}) \le \int_{x_{(i)}}^{x_{(i+1)}}
M_i(t)d t$. Since $B$ and $U$ never change signs, $\tilde{g}$ will be
monotonic and hence $c_1$-bounded for any choice of $\zeta_i$.
\paragraph{Case 2: Exactly one of $B$ and $U$ changes sign on $[x_{(i)}, x_{(i+1)}]$.}
Without loss of generality, we may assume that $g'(x_{(i)}), g'(x_{(i+1)}) \ge 0$
and that $B$ changes sign.  Consider the quantity $\phi \defeq
\int_{x_{(i)}}^{x_{(i+1)}} \max\{B(t), 0\} d t$.  If $g(x_{(i+1)}) - g(x_{(i)}) \le
\phi$, we let $m_i = B$ and $M_i = \max\{B, 0\}$.

Since, on the interval $[x_{(i)}, x_{(i+1)}]$, $B$ is piecewise linear with at most two pieces that can take on the value $0$, $B$ has at most two
roots within this interval. However, since $B(x)$ is continuous, negative
for some value of $x$, and nonnegative at $x\in\{x_{(i)},x_{(i+1)}\}$, we know $B$
has at least two roots. Thus let $r_1 < r_2$ be the roots of $B(x)$. For any
choice of $\zeta_i$, the convex combination $\zeta_i m_i + (1-\zeta_i) M_i$
will be exactly $B$ outside $(r_1, r_2)$.  Moreover, if $\zeta_i \neq 0$,
then this combination will be less than $0$ on $(r_1, r_2)$, and if $\zeta_i
= 0$, the combination will be $0$ on the whole interval. Hence it suffices
to only check the critical points $r_1$ and $r_2$.  By
(\ref{eqn:qcqp_buffer_constraints}), $m_i(r) = M_i(r) = B(r) \in
[-c_1,c_1]$ for $r\in\{r_1, r_2\}$, and so $\tilde{g}$ will be
$c_1$-bounded.

If instead $g(x_{(i+1)}) - g(x_{(i)}) > \phi$, we can recycle the argument from Case 1
with $m_i = \max\{B, 0\}$ and $M_i = U$ and conclude that $\tilde{g}$ is
$c_1$-bounded.
\paragraph{Case 3: Both $B$ and $U$ change sign on $[x_{(i)}, x_{(i+1)}]$.}
Without loss of generality, we may assume that $g'(x_{(i)}) \ge 0, g'(x_{(i+1)}) <
0$.  Since $B$ continuously interpolates between $g'(x_{(i)})$ and $g'(x_{(i+1)})$
on $[x_{(i)}, x_{(i+1)}]$, it must have a root $r$. Let $w_i\in [x_{(i)}, x_{(i+1)}]$ be
the point where $B$ changes from one linear portion to another. Then because
$B$ is monotonic on each linear portion, the fact that $B(w_i) \le
B(x_{(i+1)}) < 0$ means that $B$ cannot have a root between $[w_i, x_{(i+1)}]$
and hence has at most one root on $[x_{(i)}, x_{(i+1)}]$. Hence $r$ is the unique
root of $B$.

In a similar fashion, let us define $s$ as the root of $U$, and since $B(x)
\le U(x)$ for all $x$, we have $s \ge r$. Define
\begin{equation*}
W(x) \defeq
\begin{cases}
B(x) & x\in [x_{(i)}, r) \\
0 & x\in [r, s] \\
U(x) & t\in (s, y],
\end{cases}
\end{equation*}
and $\psi \defeq \int_{x_{(i)}}^{x_{(i+1)}} W(t) d t$. As in Case 2, we will
consider two subcases. If $g(x_{(i+1)}) - g(x_{(i)})\le \psi$, we
will let $m_i = B$ and $M_i = W$. By
(\ref{eqn:qcqp_buffer_constraints}), $m_i(r) = M_i(r) = B(r) \in [-c_1,c_1]$, and since this is the only
critical point, $\tilde{g}$ will be $c_1$-bounded.

For the other case, in which $g(x_{(i+1)}) - g(x_{(i)}) > \psi$, we choose
$m_i=W$ and $M_i=U$. Then (\ref{eqn:qcqp_buffer_constraints}) imply that
$m_i(s) = M_i(s) = U(s) \in [-c_1, c_1]$, and, since this is
the only critical point, the extension is well-defined on $(x_{(i)}, x_{(i+1)})$.
\paragraph{Defining $\tilde{g}$ outside of the interval $[x_1,x_{n'}]$}
It only remains to define our extension $\tilde{g}$ outside of the interval $[x_1, x_{n'}]$ when
either $\alpha$ or $\beta$ is infinite.
Suppose $\alpha = -\infty$.
We extend $\tilde{g}$ to each $x\in(-\infty, x_1)$ using the construction
\baligns
\tilde{g}(x) \defeq \int_{-\infty}^x \indic{t \in (x_1 - |g'(x_1)|/c_3, x_1)}
(g'(x_1) - c_3\sign(g'(x_1)) t)\ d t.
\ealigns
This extension ensures that $\tilde{g}'$ is $c_2$-bounded and $c_3$-Lipschitz.
Moreover, the constraint (\ref{eqn:qcqp_buffer_constraints})
guarantees that $|\tilde{g}(x)| \le c_1$. 
Analogous reasoning establishes an extension to $(x_{n'}, \infty)$.
\end{proof}
\section{Equivalence of Constrained Classical and Spanner Stein Discrepancies}
For $P$ with support $\xset = (\alpha_1, \beta_1)\times\cdots\times (\alpha_d, \beta_d)$ for $-\infty \le \alpha_j < \beta_j \le \infty$,
\algref{multivariate} computes a Stein discrepancy based on the graph Stein set 
\baligns
	&\gsteinset{1}{(V,E)} 
		\defeq \bigg\{ g :\xset\to\reals^d \mid 
			\forall\, x\in V,\ j,k\in\{1,\dots, d\} \text{ with } k\neq j, \text{ and } b_j \in \{\alpha_j,\beta_j\}\cap\reals, \\
			&\maxarg{
				\infnorm{g(x)}, \infnorm{\grad g(x)},
				\textstyle\frac{|g_j(x)|}{|x_{j} - b_j|},
				\textstyle\frac{|\grad_k g_j(x)|}{|x_j - b_j|},\textstyle\frac{|g_j(x) - {\grad_j g_j(x)}{(x_j - b_j)}|}{\frac{1}{2}(x_j-b_j)^2}
				} \leq 1,
			\text{  and, } \forall\, (x, y) \in E, \\
			&\maxarg{\textstyle\frac{\infnorm{g(x) - g(y)}}{\onenorm{x - y}},
				\textstyle\frac{\infnorm{\grad g(x) - \grad g(y)}}{\onenorm{x - y}},\textstyle\frac{\infnorm{g(x) - g(y) - {\grad g(x)}{(x - y)}}}{\frac{1}{2}\onenorm{x - y}^2},
				\textstyle\frac{\infnorm{g(x) - g(y) - {\grad g(y)}{(x - y)}}}{\frac{1}{2}\onenorm{x - y}^2}} \leq 1\bigg\},
\ealigns
Our next result shows that the graph Stein discrepancy based on a $t$-spanner is strongly equivalent to the classical Stein discrepancy.
\begin{proposition}[Equivalence of Constrained Classical and Spanner Stein Discrepancies]\label{prop:boundary-constraints}
If $\xset = (\alpha_1, \beta_1)\times\cdots\times (\alpha_d, \beta_d)$, and
$G_{t} = (\supp{Q}, E)$ is a $t$-spanner, then
\baligns
\langstein{Q}{\steinsetarg{1}} \leq \langstein{Q}{\gsteinset{1}{G_{t}}} \leq t^2\kappa_d\,\langstein{Q}{\steinsetarg{1}},
\ealigns
where $\kappa_d$ is a constant, independent of $(Q,P,G_{t},t)$, depending only on the dimension $d$.
\end{proposition}

\begin{proof}
\paragraph{Establishing the first inequality}
Fix any $g \in \steinsetarg{1}$, $z \in \supp{Q}$, and $j,k \in \{1,\dots,d\}$ with $k\neq j$, and consider any 
$j$-th coordinate boundary projection point
\[
b \in \{z + e_j(\alpha_j-z_j), z + e_j(\beta_j-z_j)\} \cap \reals^d.
\]
Since $b \in \partial\xset$, we must have $\inner{g(b)}{n(b)} = \inner{g(b)}{e_j} = g_j(b) = 0$.
Moreover, for each dimension $k\neq j$, we have $\grad_k g_j(x) = 0$, since otherwise, $\inner{g(b+ \delta e_k)}{n(b+ \delta e_k)}=g_j(b + \delta e_k) \neq 0$ for some $\delta\in\reals$ and $b + \delta e_k \in \partial \xset$ by the continuity of $\grad g_j$.

The smoothness constraints of the classical Stein set $\steinsetarg{1}$ now imply that
\baligns
{|g_j(z)|}
	={|g_j(z) - g_j(b)|} 
	\leq {|z_j - b_j|},\quad
{|\grad_k g_j(x)|}
	={|\grad_k g_j(z)-\grad_k g_j(b)|}
	\leq {|z_j - b_j|},
\ealigns
and 
\[
|g_j(z) - \grad_j g_j(x) (z_j-b_j)|
	= |g_j(b) - g_j(z) - \inner{\grad g_j(z)}{b-z}| 
	\leq \frac{1}{2} (z_j-b_j)^2
\]
so that all graph Stein set boundary compatibility constraints are satisfied.
Hence, we have the containment $\steinsetarg{1} \subseteq \gsteinset{1}{G_{t}}$, which implies the first advertised inequality.

\paragraph{Establishing the second inequality}
To establish the second inequality, it suffices to show that for any $\tilde{g} \in \gsteinset{1}{G_{t}}$, each $j \in \{1,\dots,d\}$,
and $\zeta \defeq t$,
there exists a function $g_j$ satisfying 
\balign
&g_j(z) = \tilde{g}_j(z),\ 
\grad g_j(z) = \grad \tilde{g}_j(z),\ 
g_j(b) = 0,\ 
\grad_k g_j(b) = 0, \;\forall k\neq j, \label{eqn:multi_boundary_C1} \\
&|g_j(b) - g_j(z)| \le \onenorm{b-z}, \label{eqn:multi_boundary_C2} \\
&\infnorm{\grad g_j(b) - \grad g_j(z)} \le \zeta \onenorm{b-z},\  
\infnorm{\grad g_j(b) - \grad g_j(b')} \le \zeta \onenorm{b-b'}, \label{eqn:multi_boundary_C4} \\
&|g_j(b) - g_j(z) - \inner{\grad g_j(z)}{b-z}| \le \frac{\zeta}{2} \onenorm{b-z}^2, \label{eqn:multi_boundary_C5} \\
&|g_j(z) - g_j(b) - \inner{\grad g_j(b)}{z-b}| \le \frac{3\zeta}{2} \onenorm{b-z}^2, \qtext{and} \label{eqn:multi_boundary_C6} \\
&|g_j(b) - g_j(b') - \inner{\grad g_j(b')}{b-b'}| \le \frac{\zeta}{2} \onenorm{b-b'}^2\label{eqn:multi_boundary_C7} 
\ealign
for all $z\in\supp{Q}$ and all $b, b'$ in the $j$-th coordinate boundary set
\[
B_j \defeq \{b \in\reals^d : b = z + e_j(\alpha_j-z_{j}) \text{ or } b = z + e_j(\beta_j-z_{j}) \text{ for some } z \in\xset\}.
\]
Indeed, since such $g_j$ will satisfy $\maxarg{|{g}_j(z)|, \infnorm{\grad {g}_j(z)}}\leq 1$ for all $z\in\supp{Q}\cup B_j$ and 
\[
\maxarg{\textstyle\frac{|g_j(x) - g_j(y)|}{\onenorm{x - y}},
				\textstyle\frac{\infnorm{\grad g_j(x) - \grad g_j(y)}}{\onenorm{x - y}},\textstyle\frac{|g_j(x) - g_j(y) - {\grad g_j(x)}{(x - y)}|}{\frac{1}{2}\onenorm{x - y}^2},
				\textstyle\frac{|g_j(x) - g_j(y) - {\grad g_j(y)}{(x - y)}|}{\frac{1}{2}\onenorm{x - y}^2}} \leq 2t^2
\]
for all ${x,y\in\supp{Q}}$
by the argument of \appref{spanner-equivalence-proof},
the Whitney-Glaeser extension theorem~\cite[Thm. 1.4]{Shvartsman08} of
\citet{Glaeser58} will then imply that there exists $g^*\in t^2\kappa_d\,\steinsetarg{1}$,
for a constant $\kappa_d$ independent of $\tilde{g}$ depending only on $d$, with $g^*(z) = g(z)$ and $\grad g^*(z) = \grad g(z)$ for all $z\in\supp{Q}$.
Since $\tilde{g}$ and $g^*$ will have matching Stein discrepancy objective values, and each objective is linear in $g$, the second advertised inequality will then follow.

Fix $\tilde{g} \in \gsteinset{1}{G_{t}}$ and $j \in \{1,\dots,d\}$.
We will now construct a function $g_j$ satisfying the desired properties.
Since
$g_j$ and $\grad g_j$ are determined on $\supp{Q}$, and $g_j$ and $\grad_kg_j$ are determined on $B_j$
for $k\neq j$ by the constraints \eqnref{multi_boundary_C1}, it remains to define $\grad_j g_j$ on $B_j$.
We choose the extension
\baligns
\grad_j g_j(b) &\defeq \min_{z\in\supp{Q}} \left \{\grad_j g_j(z) + \zeta\onenorm{z - b}\right \} \qtext{for all} b \in B_j.
\ealigns
Fix any $z\in\supp{Q}$ and $b \in B_j$, and let $b^* = z + e_j(b_j - z_{j})$.
The argument of \appref{spanner-equivalence-proof} implies that $\grad_j g_j$ is $\zeta$-Lipschitz on $\supp{Q}$, 
and hence it is also $\zeta$-Lipschitz on $\supp{Q}\cup B_{j}$.
Since 
\[
|\grad_k g_j(z) - \grad_k g_j(b)| = |\grad_k g_j(z)|\leq |z_j - b_j| \leq \onenorm{z-b}
\] for all $k\neq j$, we have (\ref{eqn:multi_boundary_C4}).
Moreover, the boundary compatibility constraints of $\gsteinset{1}{G_{t}}$ imply 
\[
|g_j(b) - g_j(z)| = |g_j(z)| \leq \onenorm{b^* - z} \leq \onenorm{b-z},
\]
establishing \eqnref{multi_boundary_C2}.
We next invoke the triangle inequality, the boundary compatibility conditions of $\gsteinset{1}{G_{t}}$, \Holder's inequality, 
the Lipschitz derivative property \eqnref{multi_boundary_C4}, and the fact $\onenorm{z-b} = \onenorm{b^*-z}+\onenorm{b^*-b}$ in turn to establish 
\eqnref{multi_boundary_C5}:
\baligns
|g_j(b) - g_j(z) - \inner{\grad g_j(z)}{b-z}|
	&= |g_j(z) - \grad_j g_j(z)(z_j-b_j) - \inner{\grad g_j(z)}{b^*-b}| \\
	&\leq |g_j(z) - \grad_j g_j(z)(z_j-b_j)| + |\inner{\grad g_j(b^*) - \grad g_j(z)}{b^*-b}| \\
	&\leq \frac{1}{2}\onenorm{b^*-z}^2 + \infnorm{\grad g_j(b^*) - \grad g_j(z)}\onenorm{b^*-b} \\
	&\leq \frac{1}{2}\onenorm{b^*-z}^2 + \zeta\onenorm{b^*-z}\onenorm{b^*-b}\\
	&\leq \frac{\zeta}{2}(\onenorm{b^*-z}+\onenorm{b^*-b})^2 
	= \frac{\zeta}{2}\onenorm{b-z}^2.
\ealigns
A parallel argument yields \eqnref{multi_boundary_C7}.
Finally, we may deduce \eqnref{multi_boundary_C6}, as
\baligns
|g_j(z) - g_j(b) - \inner{\grad g_j(b)}{z-b}| 
	&\leq |g_j(z) - \grad_j g_j(z)(z_j-b_j)| + |\grad_j g_j(b) - \grad_j g_j(z)||z_j-b_j| \\
	&\leq \frac{1}{2} (z_j-b_j)^2 + \zeta\onenorm{b-z}|z_j-b_j| 
	\leq \frac{3\zeta}{2}\onenorm{b-z}^2
\ealigns
by the triangle inequality, the definition of $\gsteinset{1}{G_{t}}$, and 
the Lipschitz property \eqnref{multi_boundary_C4}.
\end{proof}

\ifdefined\arxivmode
\subsubsection*{Acknowledgments}
The authors would like to thank Madeleine Udell for her generous advice concerning optimization in Julia, 
Quirijn Bouts and Kevin Buchin for sharing their wise counsel and greedy spanner implementation,
Francis Bach for sharing his pseudosampling code,
Andreas Eberle for his triple coupling pointers,
and 
Jessica Hwang for her feedback on various versions of this manuscript.

This work was supported by the Frederick E. Terman Fellowship and the
National Science Foundation Graduate Research Fellowship under Grant
No. DGE-114747.
{\small
\bibliography{stein}
}
\fi
\end{document}